\documentclass{ecai}
\usepackage{times}
\usepackage{graphicx}
\usepackage{latexsym}
\usepackage{amsfonts}
\usepackage{amsmath}
\usepackage{amsthm}
\usepackage{bbm}
\usepackage{pdfpages}
\usepackage[hyphens]{url}
\usepackage{algorithm}
\usepackage{algorithmic}

\newtheorem{definition}{Definition}
\newtheorem{proposition}{Proposition}
\ecaisubmission   

\begin{document}

\title{Explainable Ordinal Factorization Model: Deciphering the Effects of Attributes by Piece-wise Linear Approximation}

\author{Mengzhuo Guo\institute{Xi'an Jiaotong University, City University of Hong Kong, email: guomengzhuo@stu.xjtu.edu.cn} \and Zhongzhi Xu\institute{City University of Hong Kong, email: zhongzhxu2-c@my.cityu.edu.hk} \and Qingpeng Zhang\institute{Corresponding Author. City University of Hong Kong, email: qingpeng.zhang@cityu.edu.hk} \and Xiuwu Liao\institute{Xi'an Jiaotong University,  email: liaoxiuwu@mail.xjtu.edu.cn} \and Jiapeng Liu\institute{Xi'an Jiaotong University email: jiapengliu@mail.xjtu.edu.cn} }

\maketitle
\bibliographystyle{ecai}

\begin{abstract}
  Ordinal regression predicts the objects' labels that exhibit a natural ordering, which is important to many managerial problems such as credit scoring and clinical diagnosis. In these problems, the ability to explain how the attributes affect the prediction is critical to users. However, most, if not all, existing ordinal regression models simplify such explanation in the form of constant coefficients for the main and interaction effects of individual attributes. Such explanation cannot characterize the contributions of attributes at different value scales. To address this challenge, we propose a new explainable ordinal regression model, namely, the Explainable Ordinal Factorization Model (XOFM). XOFM uses the piece-wise linear functions to approximate the actual contributions of individual attributes and their interactions. Moreover, XOFM introduces a novel ordinal transformation process to assign each object the probabilities of belonging to multiple relevant classes, instead of fixing boundaries to differentiate classes. XOFM is based on the Factorization Machines to handle the potential sparsity problem as a result of discretizing the attribute scales. Comprehensive experiments with benchmark datasets and baseline models demonstrate that the proposed XOFM exhibits superior explainability and leads to state-of-the-art prediction accuracy.
\end{abstract}

\section{Introduction}
\label{sec-intro}
%
Ordinal regression (or ordinal classification) aims to learn a pattern for predicting the objects (e.g.: actions, items, products) labels that exhibit a natural ordering \cite{cheng2008neural}. Such problems are different from the nominal classification problems because the ordering specifies the user preferences to each object \cite{rennie2005loss}. For example, we can use an ordinal scale $\{poor, average, good, very\: good\}$ to estimate the condition of vehicles. The misclassification cost for assigning a ${good}$ vehicle to the ${poor}$ class is greater than assigning it to the ${average}$ class. Taking this ordinal information into consideration leads to more accurate models. However, if the standard nominal classification models are used without considering such information, we could obtain non-optimal solutions \cite{gutierrez2015ordinal}.

State-of-the-art methods for ordinal regression problems either transform the original problems to several binary ones or rely on the threshold-based models, which approximate a preference value for each object, and then use the trained thresholds to differentiate different classes. The ordinal regression problems play an important role in many managerial decision making problems such as clinical diagnoses \cite{bender1997ordinal}, consumer preference analysis \cite{GUO2019}, nano-particles synthesis assessment \cite{kadzinski2018co}, age estimation \cite{niu2016ordinal}, and credit scoring \cite{kim2012corporate}. In these contexts, the ability to capture the detailed relationships between the predictions and different attribute value scales in the model is as important as accuracy, because it helps the users understand and utilize the underlying model.

The existing ordinal regression methods explain the results either by providing constant coefficients measuring the relative importance of the main and interaction effects of the attributes, for instance the logistic-based models \cite{mccullagh1980regression,rennie2005loss}, or by presenting the estimated thresholds and the proportions of each possible classification, for instance the thresholds-based models \cite{sun2009kernel,gutierrez2015ordinal} and ordinal binary decomposition approaches \cite{cheng2008neural,deng2010ordinal}. Unfortunately, these methods cannot characterize the contributions of attributes at different value scales, which is critical to explaining how the model works. In addition, the boundary obtained by threshold-based methods often oversimplifies the condition, and thus may not work well for the dataset where many objects are close to the boundary.

To fill the gaps mentioned above, we propose the Explainable Ordinal Factorization Model (XOFM) for the ordinal regression problems. XOFM adopts the common assumption that there is a `preference' value for each object \cite{rennie2005loss}. XOFM uses piece-wise linear functions to approximate the actual contributions of different attribute value scales to the prediction using the training data. The trained XOFM can then estimate preference values for each object in the testing set, and calculate the probabilities of multiple relevant classes through comparing with the training samples. In this way, XOFM generalizes the thresholds from fixed values to intervals. Eventually, XOFM assigns each object the label with the highest probability. Because XOFM discretizes an attribute's value into multiple scales in the form of an attribute vector, it may lead to the sparsity problem. Thus, XOFM adopts the Factorization Machines (FMs)-based scheme to handle the sparsity problem \cite{rendle2010factorization}. The contributions are as follows:
\begin{itemize}
    \item The proposed XOFM introduces a novel ordinal regression transformation process that generalizes threshold-based ordinal regression procedures. It determines an interval for the thresholds based on the preference relationships among the objects. As a result, there is no need to initialize or estimate the single values of thresholds. 
    \item In addition to state-of-the-art prediction performance, XOFM model is able to explain the actual contributions of the main effects and interaction effects at different attribute value scales through determining the shapes of some piece-wise linear functions. Such explainability can provide detailed information to users to decipher the relationships between attributes and predictions. 
    \item We formulate the XOFM into a FMs-based scheme to handle the data sparsity problem, and extend the FMs to handle ordinal regression problems. To our best of knowledge, this is the first study that enhances the explainability of FMs in performing ordinal regression tasks.  
\end{itemize}

\section{Related Work}
\label{sec-relatedwork}
\subsection{Ordinal regression}
\label{subsec-ordre}

The ordinal regression approaches are commonly classified into three groups: \textit{na\"ive approaches}, \textit{ordinal binary decomposition approaches} and \textit{threshold-based models} \cite{gutierrez2015ordinal}. The na\"ive approaches either do not consider the preference levels of classes (such as using the standard nominal classifiers), or map the class labels into real values (such as the support vector regression, SVR model) \cite{smola2004tutorial}. However, the mapping of the class labels may hinder the performances because the metric distances between ordinal scales are usually unknown \cite{sun2009kernel}. 

Ordinal binary decomposition approaches solve the problem by decomposing an original ordinal regression problem into several binary classification problems. \cite{cheng2008neural} proposed a neural network-based method for ordinal regression (NNOR). The method decomposes and encodes the class labels, then trains a single neural network model to classify the objects. Similarly, extreme learning machine (ELMOR), a single-layer feed-forward neural network-based model, has also been adapted to ordinal regression problems \cite{deng2010ordinal}. More recently, CNNOR, a convolutional neural network (CNN)-based model was proposed to handle the ordinal regression problems. CNNOR is unique in being able to handle small datasets given the CNN structure \cite{liu2017deep}. These methods achieve good performance, but are limited in model explainability given their neural network scheme.

The threshold-based models are the most popular approaches for ordinal regression problems. They assume that there is a function measuring the `preference' values of the objects and compare the preference values to a set of thresholds, which are either predefined or estimated using data. Following this common assumption, the proportional odds model (POM) uses a standard logistic function to predict the probability of an object being classified to a class \cite{mccullagh1980regression}. POM has become the standard ordinal regression method, and the basis for most followup threshold-based models \cite{gutierrez2015ordinal}. Typical examples include the ordinal logistic model with immediate-threshold variant (LIT) and all-threshold variant (LAT) \cite{rennie2005loss}, and kernel discriminant learning for ordinal regression (KDLOR) \cite{sun2009kernel}. Unfortunately, none of these models can decipher the contributions of the attributes at different specific value scales.

\subsection{Factorization Machines}
\label{subsec-FM}

Factorization machines (FMs) combine the advantages of support vector machines with factorization models. FMs factorize the interaction parameters instead of directly estimating them, thus they have the advantage in handling sparse data \cite{rendle2010factorization}. FMs have been widely applied to various machine learning problems including recommender systems \cite{rendle2011fast}, click predictions \cite{wang2017deep}, and image recognition \cite{li2017factorized}. FMs can be used for ranking problems given the aforementioned characteristics. However, few studies applied FMs to ordinal regression due to the difficulty in transforming the ordinal regression problems into a ranking form. Ordinal factorization machine and hierarchical sparsity (OFMHS) utilizes FMs for ordinal regression through formulating it as a convex optimization problem. In particular, OFMHS can model the hierarchical structure behind the input variables \cite{guoordinal}. Although the method achieves state-of-the-art performance, it does not explore the detailed relationship between the attributes and predictions.

\subsection{Explainable Models}
\label{subsec-expmodel}

Apparently there are various definitions of the model explainability. Here we focus on enhancing the \textit{model induction} as defined in \cite{gunning2017explainable}. Specifically, an explainable model should be able to characterize the contributions of the individual attributes and reveal the interaction effects of the attributes \cite{lou2012intelligible}. They used \textit{score/shape functions} to measure the main effects and mapped the interaction effects to real values \cite{lou2013accurate}. At last, different \textit{link functions} are used to link these score functions with various machine learning tasks \cite{vaughan2018explainable,tsang2018neural}. Unfortunately, these powerful explainable models cannot be directly adopted for ordinal regression problems. In this study, we use the piece-wise linear score functions to characterize the contributions of individual attributes. The mapping functions explore the interaction effects of the attributes, and the link function measures the `preference' values of the objects.

\section{Explainable Ordinal Factorization Model}
\label{sec-FMmodel}
\subsection{Preliminaries}
\label{subsec-pre}

Consider an ordinal regression problem that concerns a set of $N$ objects $\mathcal{X} = \{\mathbf{x}_1,\dots,\mathbf{x}_i,\dots,\mathbf{x}_N \}$, where $\mathbf{x}_i = (x_{i,1},\dots,x_{i,j},\dots,x_{i,m}) \in \mathbb{R}^m$, and the class label $y_i\in \mathcal{Y} = \{1,\dots,h,\dots,H \}$. The classes are in a natural ordering $C_H \succ C_{H-1} \succ \cdots \succ C_1$, where $C_h \succ C_{h-1}, h = 2,\dots,H$ indicates that the objects in the class $C_h$ are preferred to those in $C_{h-1}$. An attribute interaction, $\mathcal{I}$, is a subset of all attributes: $\mathcal{I}\subseteq \{1,2,\dots,m \}$. We denote $\mathcal{I}_d \subseteq \mathcal{I}, d\in D=\{2,3,\dots,m \}$ be the set of all $d$-order interactions. More specifically, if $d=2$, $\mathcal{I}_2=\{ \{i,j\} | i,j\in D, i\neq j \}$ is the set of all pairwise interactions. 

We assume the `preference' value of each object is determined by a link function:
\begin{align}
    U(\mathbf{x}_i) &= \sum\limits_{j = 1}^m {{u_j}\left( { x_{i,j} } \right)} + \sum\limits_{\{ {j_1},{j_2}\}  \subseteq {\mathcal{I}_2}} {{u_{{j_1},{j_2}}}\left( {{x_{i,j_1}, x_{i,j_2}}} \right)}\nonumber  \\ &+  \cdots + \sum\limits_{\{ {j_1}, \ldots {j_d}\}  \subseteq {\mathcal{I}_d}} {{u_{{j_1}, \ldots {j_d}}}\left( {x_{i,j_1}, \ldots ,{x_{i,j_d}}} \right)}\label{eq-globalvalue}
\end{align}

where $d \in D$ is the predefined highest order of interactions. $u_j(\cdot), j=1,\dots,m$ denotes score function of $j$-th attribute and $u_{j_1,\cdots,j_d}(\cdot), d\in D$ denotes a mapping function of the interacting attributes. We use a piece-wise linear function to estimate each score function because any nonlinear functions can be approximated by sampling the curves and interpolating linearly between the points \cite{hamann1994data}. Let $\mathbf{X}_j = \left[\alpha_j, \beta_j \right]$, where $\alpha_j = min\{x_{i,j}|\mathbf{x}_i \in \mathcal{X} \}$ and $\beta_j = max\{x_{i,j}|\mathbf{x}_i \in \mathcal{X} \}$, be the whole value scale of the $j$-th attribute. For each attribute, we partition the scale into $\gamma_j$ equal sub-intervals $[\varphi^0_j, \varphi^1_j], [\varphi_j^1, \varphi_j^2 ],\dots,[\varphi_j^{\gamma_j-1}, \varphi_j^{\gamma_j} ]$, where $\varphi_j^k = {\alpha _j} + \frac{k}{{{\gamma _j}}}\left( {{\beta _j} - {\alpha _j}} \right),k = 0, 1, \ldots {\gamma _j}$ are called characteristic points.

\begin{definition}
The attribute vector $\boldsymbol{\Phi}_i\in \mathbb{R}^\gamma,\gamma={\sum\nolimits_{j = 1}^m {{\gamma _j}} }$ of object $\mathbf{x}_i \in \mathcal{X}$ is defined as follows:
\begin{equation}
    \boldsymbol{\Phi}_i = {\left( {\underbrace {\phi_{i,1}^1, \ldots \phi_{i,1}^{{\gamma _1}},}_{\text{$1$-st \: Attribute}} \ldots, \underbrace {\ldots, \phi_{i,j}^{k_j},  \ldots,}_{\text{$j$-th \: Attribute}} \ldots \underbrace {,\ldots \phi_{i,m}^{{\gamma _m} }}_{\text{$m$-th \: Attribute}}} \right)^T}
\end{equation}
where $\phi_{i,j}^{k_j} = \left\{ \begin{array}{l}
1, \quad x_{i,j} > \varphi_j^{{k_j}},\\
\frac{{x_{i,j} - \varphi_j^{k_j-1}}}{{\varphi_j^{{k_j} } - \varphi_j^{k_j-1}}}, \quad \varphi_j^{k_j-1} \le x_{i,j} \le \varphi_j^{{k_j}},\\
0,\quad \quad \text{otherwise}.
\end{array} \right. j=1,\dots,m,$ and ${k_j}=1,\dots,\gamma_j.$ \label{def-1}
\end{definition}

\begin{definition}
The marginal score vector $\mathbf{u}$ is defined as:
\begin{equation}
    \mathbf{u}={\left( {\Delta _1^1, \ldots ,\Delta _1^{{\gamma _1}}, \ldots ,\Delta _j^{{k_j}}, \ldots ,\Delta _m^1, \ldots ,\Delta _m^{{\gamma _m}}} \right)^T}
\end{equation}
where $\Delta_j^{k_j} = u_j(\varphi_j^{k_j})-u_j(\varphi_j^{k_j-1}),j=1,\dots,m,$ and $k_j = 1,\dots,\gamma_j$, is the difference between two consecutive characteristic points. \label{def-2}
\end{definition}

Given Definitions \ref{def-1} and \ref{def-2}, the first term in Eq.(\ref{eq-globalvalue}), namely the main effects of the attributes, can be reformulated. As for the interaction effects part, we consider the pairwise interactions ($d=2$) because this is the most common situation. Since directly estimating the individual mappings of $u_{j_1,j_2}(\cdot) \in \mathbb{R}, \{j_1,j_2\} \subseteq \mathcal{I}_2$ may cause data sparsity problem \cite{rendle2010factorization}, we model these interactions by factorizing them. The new link function is as follows:
\begin{equation}
    U(\mathbf{x}_i) = \mathbf{u}^T\boldsymbol{\Phi}_i + \sum\limits_{j_1 = 1}^\gamma  {\sum\limits_{j_2 = j_1 + 1}^\gamma  {\left\langle {{\mathbf{v}_{j_1}},{\mathbf{v}_{j_2}}} \right\rangle {\Phi_{i,{j_1}}}{\Phi_{i,j_2}}} } 
    \label{eq-FMglobal}
\end{equation}
where $\left\langle {{\mathbf{v}_{{j_1}}},{\mathbf{v}_{{j_2}}}} \right\rangle = \sum\limits_{f = 1}^k {{v_{{j_1},f}}{v_{{j_2},f}}} $ is a dot product of size $k$, and $\Phi_{i,j}$ is the $j$-th element in vector $\boldsymbol{\Phi}_i$. The Eq.(\ref{eq-FMglobal}) is in the form of FMs. The only difference is that the Eq.(\ref{eq-FMglobal}) does not have a global bias term, which does not affect the predictions (please refer to the following section).

\subsection{Transform Ordinal Regression Problems}
\label{subsec-trans}
We introduce a new process for transforming ordinal regression problems. The process determines the labels of the objects in the testing set by comparing their `preference' values to the objects in the training set $\mathcal{X}_{train} \subseteq \mathcal{X}$ \cite{greco2010multiple}. The boundaries of each class are learned from these comparisons and are defined by intervals instead of some fixed real values, which generalizes the threshold-based ordinal regression procedures.

\begin{definition}
Given a link function $U(\cdot)$, the `preference' relationship between two objects is concordance with $U$ if and only if:
\begin{equation}
    \forall \mathbf{x}_i,\mathbf{x}_j \in \mathcal{X}_{train}, U(\mathbf{x}_i)\ge U(\mathbf{x}_j) \Longrightarrow y_i \ge y_j
\end{equation}
and equivalently, 
\begin{equation}
    \forall \mathbf{x}_i,\mathbf{x}_j \in \mathcal{X}_{train}, y_i > y_j \Longrightarrow U(\mathbf{x}_i) \ge U(\mathbf{x}_j) + \tau \label{eq-constraint}
\end{equation} where $\tau$ is a predefined positive margin. \label{def-3}
\end{definition}

Definition \ref{def-3} assumes the labels of objects are concordant to their `preference' values (scores), i.e., the greater the value of $U(\mathbf{x}_i)$, the more likely that the object $\mathbf{x}_i$ being assigned to better classes. Such information is helpful for constructing loss functions for ordinal regression problems. Given the following definition, we determine the multiple relevant classes for objects $\mathbf{x}_i \in \mathcal{X}$:
\begin{definition}
Given the link function $U(\cdot)$, the class interval of an object $\mathbf{x}_i \in \mathcal{X}$ is $[C_{L_i}, C_{R_i}]$, where:
\begin{equation}
    L_i = Max( \{ 1 \} \cup \{y_j: U(\mathbf{x}_j) \le U(\mathbf{x}_i), \mathbf{x}_j \in \forall \mathcal{X}_{train}  \} )
    \label{eq-L}
\end{equation}
and 
\begin{equation}
    R_i = Min( \{ H \} \cup \{y_j: U(\mathbf{x}_j) \ge U(\mathbf{x}_i), \mathbf{x}_j \in \forall \mathcal{X}_{train}  \} )
    \label{eq-R}
\end{equation}
\label{def-4}
\end{definition}

\begin{proposition}
Given the link function $U(\cdot)$, the interval $[C_{L_i}, C_{R_i}]$ of object $\mathbf{x}_i \in \mathcal{X}$ is not empty. \label{prop-1}
\end{proposition}

\proof
If $set_1 = \{\mathbf{x}_j \in \forall \mathcal{X}_{train}: U(\mathbf{x}_j) \le U(\mathbf{x}_i)\} = \emptyset$, then $L_i = 1$, thus $L_i \le R_i$ and the interval is not empty. Analogously, if $set_2 = \{\mathbf{x}_j \in \forall \mathcal{X}_{train}: U(\mathbf{x}_j) \ge U(\mathbf{x}_i)\} = \emptyset$, then $R_i = H$, thus $L_i \le R_i$ and the interval is not empty.

We prove it when $set_1 \ne \emptyset$ and $set_2 \ne \emptyset$ by contradiction. Assume $L_i > R_i$, we have $\mathbf{x}_j \in set_1$ and $\mathbf{x}_k \in set_2$ such that $y_j > y_k$. As stated in Definition \ref{def-3}, $y_j > y_k$ indicates that $U(\mathbf{x}_j)\ge U(\mathbf{x}_k) + \tau$. Since $\tau > 0$, thus $U(\mathbf{x}_j) > U(\mathbf{x}_k)$. Note that $\mathbf{x}_j \in set_1$ and $\mathbf{x}_k \in set_2$, thus $U(\mathbf{x}_j) \le U(\mathbf{x}_k)$, which contradicts the assumption and concludes the proof.
\endproof

Given Definition \ref{def-3} and Proposition \ref{prop-1}, XOFM always provides an interval for each object. The interval contains either a single class or multiple relevant classes. The users can determine the final class based on either their domain knowledge or the following indicator. We define an indicator $\kappa(\mathbf{x}_i \to C_h)$ that favors the classification $\mathbf{x}_i \to C_h, h=1,\dots,H$:
\begin{definition}
Given an class interval $[C_{L_i}, C_{R_i}]$, if ${L_i} = {R_i}$, then $\mathbf{x}_i \to C_{L_i}$. If not, the indicator $\kappa(\mathbf{x}_i \to C_h) = \frac{ card(\hat{y}_i^h)}{{\sum\nolimits_{s = 1,...,h - 1,h + 1,...,H} {\left| {{\mathcal{X}_s}} \right|} }}, h = L_i,L_i+1,\dots,R_i$, where ${\mathcal{X}_s} = \{\mathbf{x}_j\in \mathcal{X}_{train}|y_j=s \}, s = 1,\dots,h-1,h+1,\dots,H$ and
\begin{align}
    card(\hat{y}_i^h) &= {\sum\limits_{s = 1,...,h - 1} {\left| {\left\{ {{\mathbf{x}} \in {\mathcal{X}_s}|U({\mathbf{x}_i}) - U({\mathbf{x}}) > \tau } \right\}} \right|} } \nonumber \\
    &{ + \sum\limits_{s = h + 1,...,H} {\left| {\left\{ {{\mathbf{x}} \in {\mathcal{X}_s}|U({\mathbf{x}}) - U({\mathbf{x}_i}) > \tau } \right\}} \right|} } \nonumber
\end{align}
\label{def-5}
\end{definition}

Definition \ref{def-5} provides a proportion of objects that are classified to a class either worse or better than $C_h$. Obviously, the greater $\kappa(\mathbf{x}_i \to C_h)$ is, the more likely that $\hat{y}_i = h$. Hence, we classify $\mathbf{x}_i$ to the class with the maximal $\kappa(\mathbf{x}_i \to C_h)$.
\begin{proposition}
The proposed transformation procedure for ordinal regression generalizes the threshold-based procedure. The classifications determined by the proposed procedure can be obtained by fixing the thresholds within some intervals instead of single values. 
\end{proposition}
\begin{proof}
From Definition \ref{def-3}, the `preference' between two objects are concordance with the value of link function $U(\cdot)$, i.e., $\mathbf{x}_i^*,\mathbf{x}_j^*\in \mathcal{X}_{train}, U(\mathbf{x}_i^*) \geq U(\mathbf{x}_j^*) \Rightarrow y_i^* \geq y_j^*$. Assume the threshold $b_h$ and $b_{h-1}$ are the upper and lower bounds of $C_h,h\in \mathcal{Y}$, respectively. If $Max\{U(\mathbf{x}_i^*)|{\mathbf{x}_i^*\in \mathcal{X}_{h}}\} < b_h \leq Min\{U(\mathbf{x}_i^*)|{\mathbf{x}_i^*\in\mathcal{X}_{h+1}}\}$, then any object $\mathbf{x}_i \in \mathcal{X}$ such that $b_{h-1}\leq U(\mathbf{x}_i) < b_h$ will be classified to class $C_h$ by threshold-based procedure. Considering testing samples $\mathbf{x}_i\in \mathcal{X}/\mathcal{X}_{train}$, we have two cases:

(1) $\exists h \in \{1,\dots, H\}$ and $Min\{U(\mathbf{x}_i^*)|{\mathbf{x}_i^*\in \mathcal{X}_{h}}\}\leq U(\mathbf{x}_i) < Max\{U(\mathbf{x}_i^*)|{\mathbf{x}_i^*\in \mathcal{X}_{h}}\}$. Given Definition \ref{def-4}, $L_i = R_i = h$ and the proposed procedure will classify $\mathbf{x}_i$ to class $C_h$. Moreover, since $Min\{U(\mathbf{x}_i^*)|{\mathbf{x}_i^*\in\mathcal{X}_{h}}\}\leq U(\mathbf{x}_i) <Max\{U(\mathbf{x}_i^*)|{\mathbf{x}_i^*\in\mathcal{X}_{h}}\}$, the threshold-based procedure will always classifies $\mathbf{x}_i$ to $C_h$ as long as the thresholds $b_{h-1}$ and $b_h$ satisfying $Max\{U(\mathbf{x}_i^*)|{\mathbf{x}_i^*\in \mathcal{X}_{h-1}} \} <  b_{h-1} \leq Min\{U(\mathbf{x}_i^*)|{\mathbf{x}_i^*\in\mathcal{X}_{h}} \}$ and $Max\{U(\mathbf{x}_i^*)|{\mathbf{x}_i^*\in \mathcal{X}_{h}} \} < b_h$ $\leq Min\{U(\mathbf{x}_i^*)|\mathbf{x}_i^*\in\mathcal{X}_{h+1}\}$.
 
(2) $\exists h \in \{1,\dots, H-1\}$ and $Max\{U(\mathbf{x}_i^*)|{\mathbf{x}_i^*\in \mathcal{X}_{h}}\} < U(\mathbf{x}_i) < Min\{U(\mathbf{x}_i^*)|{\mathbf{x}_i^*\in \mathcal{X}_{h+1}}\}$. Given Definition \ref{def-4}, $L_i =h $ and $ R_i = h+1$, the proposed procedure will classify $\mathbf{x}_i$ to an interval $[C_h,C_{h+1}]$. Moreover, if $ Max\{ U(\mathbf{x}_i^*)|{\mathbf{x}_i^*\in \mathcal{X}_{h}} \} < b_h \leq U(\mathbf{x}_i)$, the threshold-based procedure will classify $\mathbf{x}_i$ to $C_{h+1}$, and if $U(\mathbf{x}_i)< b_h \leq Min\{U(\mathbf{x}_i^*)|{\mathbf{x}_i^*\in \mathcal{X}_{h+1}} \}$ the classification will be $C_h$. Given two cases, the threshold-based procedure provides single values of the thresholds, and classifies an object to a single class within an interval that can be stemmed from the proposed procedure. Therefore, the propose procedure is a general form of the threshold-based procedure. 
\end{proof}

\subsection{Learning XOFM}
\label{subsec-train}
The parameters in the proposed XOFM, i.e., $\mathbf{u}\in \mathbb{R}^\gamma, \mathbf{V} \in \mathbb{R}^{\gamma \times k}$ can be estimated under a standard FMs-based scheme with the following loss function:
\begin{equation}
    Loss = \frac{1}{2}\sum\limits_{{\mathbf{x}_i},{\mathbf{x}_j} \in \mathcal{X}_{train},{\mathbf{x}_i} \ne {\mathbf{x}_j}} {l\left( {\left( {{\mathbf{x}_i},{\mathbf{x}_j}} \right),{y_{i,j}}} \right)^2} 
\end{equation}
where $l\left( {\left( {{\mathbf{x}_i},{\mathbf{x}_j}} \right),{y_{i,j}}} \right) = \max \left\{ {0,U({\mathbf{x}_j}) - U({\mathbf{x}_i}) + \tau } \right\}$ if ${y_{i,j}} = {y_i} - {y_j} > 0$. All pairwise comparisons with $U({\mathbf{x}_i}) - U({\mathbf{x}_j})$ less than the predefined margin $\tau$ are penalized. Obviously, since we focus on the difference between two objects scores, the global bias term in traditional FMs can be discarded. The model parameters can be learned by gradient descent methods. The gradient of the parameters in XOFM is \cite{rendle2012factorization}:
\begin{equation}
    \frac{\partial U({\mathbf{x}_i})}{{\partial \theta }} = \left\{ \begin{array}{l}
{\Phi _{i,j}}, \quad \quad \quad \quad \quad \quad \quad \quad \quad \quad \quad \text{if}\: \theta\: \text{is}\:{\mathbf{u}_j}\\
{\Phi _{i,j}}\sum\nolimits_{k = 1}^m {{v_{k,f}}{\Phi _{k,j}}}  - {v_{j,f}}\Phi _{i,j}^2,\:\text{if}\: \theta \: \text{is}\:{v_{j,f}}
\end{array} \right.
\end{equation}

For direct optimization of the loss function, the derivatives are:
\begin{align}
    &\frac{\partial }{{\partial \theta }}l{\left( {\left( {{{\mathbf{x}}_i},{{\mathbf{x}}_j}} \right),{y_{i,j}}} \right)^2} = \nonumber \\ &\left\{ \begin{array}{l}
0,\quad \quad \quad \quad \quad \quad \quad \text{if} \quad U({\mathbf{x}_i}) - U({\mathbf{x}_j}) - \tau  \ge 0\\
2(U({\mathbf{x}_j}) - U({\mathbf{x}_i}) + \tau )(\frac{{\partial U({\mathbf{x}_j})}}{{\partial \theta }} - \frac{{\partial U({\mathbf{x}_i})}}{{\partial \theta }}) \quad \text{o.w.}
\end{array} \right.
\end{align}

The computational complexity for the training process is linear while the computational complexity for preprocessing the data is $\mathcal{O}(N^2)$. We can use the standard optimization algorithms that have been proposed for other machine learning models to estimate the parameters in XOFM. Stochastic gradient descent (SGD) is an iterative method for optimizing an objective function with smoothness properties \cite{robbins1951stochastic}. Since the loss function in XOFM is convex, it is suitable to use SGD algorithm to optimize the parameters. Nevertheless however, other complex optimization algorithms such as advanced stochastic approximation algorithms and Markov Chain Monte Carlo inference are also suitable for the proposed XOFM. Algorithm \ref{alg-sgd} shows how to apply SGD to optimize the XOFM.

\begin{algorithm}[h]
\caption{SGD for XOFM.}
\label{alg-sgd}
\begin{algorithmic}[1]
\REQUIRE Training data $\mathbf{x}_i,\mathbf{x}_i \in \mathcal{X}_{train}$, regularization terms $\lambda_1,\lambda_2$, predefined number of sub-intervals $\gamma_j,j=1,\dots,m$, dot product size $k$, learning rate $\eta$, number of iteration $iter$, and initialization $\sigma$. \\
\ENSURE Parameters $\mathbf{w}$ and $\mathbf{V}$ in XOFM. \\
\STATE Vectorize the attributes scales and determine the differences of the attributes vectors between pairwise alternatives.\\
\STATE Initialize $\mathbf{w}\leftarrow \left(0,\dots,0 \right),\mathbf{V}\sim \mathcal{N}\left(0,\sigma \right)$.\\
\WHILE{$i \leq iter$}
\FOR{$\left( \mathbf{x}_i,\mathbf{x}_j\right) \in \{\left( \mathbf{x}_i,\mathbf{x}_j\right)|\mathbf{x}_i,\mathbf{x}_j \in \mathcal{X}_{train}, y_i>y_j \}$}
\FOR{$n \in \{1,\dots,\gamma \}$}
\STATE $w_n \leftarrow w_n - \eta \left( \frac{\partial }{{\partial w_n }}l{\left( {\left( {{{\mathbf{x}}_i},{{\mathbf{x}}_j}} \right),{y_{i,j}}} \right)^2} + 2\lambda_1 w_n \right)$;
\FOR{$f \in \{1,\dots,k \}$}
\STATE $v_{n,f} \leftarrow v_{n,f} - \eta \left( \frac{\partial }{{\partial v_{n,f} }}l{\left( {\left( {{{\mathbf{x}}_i},{{\mathbf{x}}_j}} \right),{y_{i,j}}} \right)^2} + 2\lambda_2 v_{n,f} \right)$;
\ENDFOR
\ENDFOR
\ENDFOR
\STATE $i\leftarrow i+1$;
\ENDWHILE
\end{algorithmic}
\end{algorithm}

\subsection{Explainable Model and Regularization}
\label{subsec-expandreg}

XOFM uses piece-wise linear functions to approximate the actual contributions of the attributes at different value scales. More specifically, the vector $\mathbf{u}$ characterizes the main effects by presenting a score function versus individual attributes at different value scales. The parameters in matrix $\mathbf{V}$ decipher the interaction effects of discretized value scales between two attributes. These parameters can help us understand the pairwise interactions via a interaction matrix (visualized as a heat-map), in which the whole area is divided into small blocks and each block represents the interaction effects of the corresponding intervals of attribute scales. 

To avoid the over-fitting problem, we can modify the loss function of XOFM by adding regularization terms for both the main and interaction effects:
\begin{equation}
    NewLoss = Loss + {\lambda _1}\left\| \mathbf{u} \right\|_2^2 + {\lambda _2}\left\| \mathbf{V} \right\|_F^2
\end{equation}
where $\left\| \cdot \right\|_F $ is \textit{Frobenius} norm. In addition to avoid over-fitting problem, the regularization terms can constrain the shape of score functions and adjust the effect of the attribute interactions. Obviously, $\lambda_1$ determines the complexity of the score functions of individual attributes. When $\lambda_1$ increases, the model tends to penalize the difference between the two consecutive characteristic points, thus the score functions change smoother. In contrast, $\lambda_2$ determines the impact of the attribute interactions on the link function. A smaller $\lambda_2$ leads to less intensity of attribute interactions. The value of $\lambda_1$ and $\lambda_2$ can be predefined in accordance with the users' domain knowledge. For instance, if the user ensures that the involved attributes are usually irrelevant to each other, $\lambda_2$ can be set to a large value.

The capacity to explain the results makes XOFM helpful for managerial problems. For example, physicians need accurate re-admission prediction models that can reveal the detailed effects of risk factors for individual patients. By visualizing the main and interaction effects of the risk factors, it is easier for physicians to examine the consistency between the underlying model and their prior knowledge. Such explainability is important to test the rationality of a prediction model. 

\subsection{XOFM with monotonicity constraints}

In some real world decision problems, the user's prior knowledge about some attributes, for instance the monotonicity of attributes is required to be satisfied \cite{jacquet2001preference}. The proposed XOFM can also adapt to these cases where the attributes are restricted to be monotonic. Note that in these cases, the monotonicity of interaction effects of two monotonic attributes should also be maintained \cite{greco2014robust}. For this purpose, we reformulate original XOFM with additional constraints:

\begin{equation}
\begin{array}{l}
\mathop {Min}\limits_{{\bf{u}},{\bf{V}}}  \: NewLoss\\
s.t.:\Delta _j^s \ge 0,s = 1, \ldots {\gamma _j},j = 1, \ldots m,\\
\quad \quad \left\langle {{{\bf{v}}_{{n_1}}},{{\bf{v}}_{{n_2}}}} \right\rangle  \ge 0,{n_1} = 1, \ldots \gamma ,{n_2} = 1, \ldots \gamma ,
\end{array}
\label{eq-newmodel}
\end{equation}

There are many methods can be used to optimize Problem (\ref{eq-newmodel}). In this study, we first substitute $\Delta_j^k$ by $({\Delta'}_j^{k})^2$ and do not constrain on ${\Delta'}_j^{k}$. In this way, the problems regarding ${\Delta'}_j^{k}$ are unconstrained, thus standard gradient decent algorithms can be used for optimizing ${\Delta'}_j^{k}$ and $\Delta_j^k = ({\Delta'}_j^{k})^2$. We then use the projected gradient methods to optimize the parameters in $\mathbf{V}$ \cite{luenberger1984linear}. More specifically, the constraints on dot products can be replaced with the following constraints regarding $v_{n,f}$:
\begin{equation}
\begin{array}{l}
\mathop {Min}\limits_{{\bf{u}},{\bf{V}}} \: NewLoss\\
s.t.:\Delta _j^s = {\left( {\Delta _j^{'s}} \right)^2},s = 1, \ldots {\gamma _j},j = 1, \ldots m,\\
{\rm{      }}{v_{{n},{f}}} \ge 0, n = 1,\ldots,\gamma, f=1,\ldots,k,
\end{array}
\label{eq-newmodel2}
\end{equation}
where $v_{{n},{f}}$ is the $f$-th element in vector $\mathbf{v}_n$. Note that the feasible region in Problem (\ref{eq-newmodel2}) is convex subset of the feasible region in Problem (\ref{eq-newmodel}). Although the new constraints are stricter, they are simpler to solve \cite{xu2018modeling}. To solve the Problem (\ref{eq-newmodel}), we define an indicator function: $\mathbbm{1}_{\infty}(\mathbf{v}_n)=\left\{ \begin{array}{l}
0,\: v_{{n,f}}\ge 0, \forall {v_{{n,f}}}\: \text{in} \: {{\bf{v}}_n}\\
\infty,\: \text{otherwise}
\end{array} \right.$ and rewrite Problem (\ref{eq-newmodel2}) as:
\begin{equation}
\mathop {Min}\limits_{{\bf{u}},{\bf{V}}}  \: NewLoss' + \sum\limits_{n = 1}^\gamma  {{\mathbbm{1}_{\infty}(\mathbf{v}_n)}}
\label{eq-newmodel3}
\end{equation}
Note that the ${\Delta}_j^{s}$ is replaced by $({\Delta'}_j^{s})^2$ in $NewLoss'$. in Problem (\ref{eq-newmodel3}), we first optimize the differentiable part $NewLoss'$ in the objective function and then use an euclidean projection to ensure that the solutions are in the feasible region. We add an extra step between steps 8 and 9 in Algorithm \ref{alg-sgd}, i.e., $v_{n,f}\leftarrow P(v_{n,f})$, where $P(v_{n,f}) = \left\{ \begin{array}{l}
0,\: {v_{n,f}} \le 0,\\
{v_{n,f}},\: \text{otherwise},
\end{array} \right.$ for all $n=1,\ldots,\gamma,f=1,\ldots,k$.

We apply this specific model to a real dataset and the experimental results, including the obtained marginal value functions and all pairwise interaction effects are presented in the online version.

\section{Experimental Analysis}
\label{sec-exper}

\subsection{Experimental Design}
\label{subsec-expdes}

We evaluate the proposed XOFM on seven benchmark datasets and compare its performance with that of state-of-the-art baseline ordinal regression methods. The characteristics of the datasets\footnote{Abalone, Auto Riskiness, Boston Housing, Stock and Skill are downloaded from \url{https://www.gagolewski.com/resources/data/ordinal-regression/}, Breast data is downloaded from UCI \cite{Dua:2019}, and Chinese University data is collected from \url{http://www.shanghairanking.com/Chinese_Universities_Rankings/}.} are presented in Table \ref{table-dataset} and the selected baselines are presented in Table \ref{table-orapproach}. A five-fold cross validation process is used to train the models. Note that for consistency, we do not add any regularization terms in this experiment and the parameters in XOFM are $\gamma$ and $\tau$. We use the standard stochastic gradient descent algorithm to optimize XOFM, but other algorithms introduced in \cite{rendle2012factorization} can also be applied to optimize XOFM. After the parameters in each model are determined, we randomly select 80\% and 20\% of the data as the training set and testing set, respectively, and then average the results over 30 trials to evaluate the performance. The code for the proposed XOFM is attached for review purpose, and will be made publicly available on Github. 

\begin{table}
\begin{center}
{\caption{Characteristics of the datasets.}\label{table-dataset}}
\begin{tabular}{llll}
\hline
Dataset & \#Obj. & \#Attr.& \#Classes \\
\hline
abalone ord (AO) & 4,177 & 7 & 8 \\
auto riskiness (AR) & 160 & 15 & 6 \\
breast (BR) & 106 & 9 & 6\\
Boston housing ord (BHO) & 506 & 13 & 5\\
Chinese university (CU) & 600 & 10 & 5 \\
stock ord (SO)& 950 & 9 & 5\\
skill (SK) & 3,302 & 18 & 6 \\
\hline
\end{tabular}
\end{center}
\end{table}

\begin{table}
\begin{center}
{\caption{Selected ordinal regression approaches.}\label{table-orapproach}}
\begin{tabular}{ll}
\hline
Abbr. & Short description \\
\hline
LIT & Ordinal logistic model with immediate-threshold variant \cite{rennie2005loss}. \\
LAT & Ordinal logistic model with all-threshold variant \cite{rennie2005loss}. \\
KDLOR & Kernel discriminant learning for ordinal regression \cite{sun2009kernel}. \\
POM & Proportional odds model \cite{mccullagh1980regression}. \\
SVR & Support vector regression \cite{smola2004tutorial}.\\
NNOR & Neural network with ordered partition for ordinal regression \cite{cheng2008neural}. \\
ELMOR & Extreme learning machine for ordinal regression \cite{deng2010ordinal}. \\
CNNOR & Convolutional deep neural network for ordinal regression \cite{liu2017deep}. \\
OFMHS & Ordinal factorization machine with hierarchical sparsity \cite{guoordinal}.\\
XOFM & Proposed explainable ordinal factorization model.\\
\hline
\end{tabular}
\end{center}
\end{table}

To evaluate the performance, we adopt two measures. First, we use the \textit{Accuracy (Acc)} to measure the global performance but does not consider the order:
\begin{equation}
    Acc = \frac{1}{N}\sum_{i=1}^{N} \left[\kern-0.15em\left[ {\hat{y}_i = y_i}  \right]\kern-0.15em\right]
\end{equation}
where $\hat{y}_i$ is the predicted label. 

Second, we adopt the \textit{Mean Absolute Error} ($MAE$) to measure the deviation of the predicted labels from the actual labels \cite{baccianella2009evaluation}:
\begin{equation}
    MAE = \frac{1}{N}\sum_{i=1}^{N}|\hat{y}_i - y_i|
\end{equation}

\subsection{Results Analysis}
\label{subsec-result}
\begin{table*}
\begin{center}
{\caption{Test $Acc$ results for each dataset and approach.}\label{table-accresult}}
\resizebox{0.95\textwidth}{!}{ 
\begin{tabular}{llllllll}
\hline
Approaches & AO & AR & BR & BHO & CU & SO & SK \\
\hline
LIT & 0.304$\pm$0.034 & 0.469$\pm$0.031 & 0.619$\pm$0.015 & 0.634$\pm$0.024 & 0.783$\pm$0.019 & 0.721$\pm$0.023 & 0.325$\pm$0.086 \\
LAT & 0.316$\pm$0.037 & 0.219$\pm$0.034 & 0.619$\pm$0.014 & 0.634$\pm$0.021 & 808$\pm$0.021 & 0.710$\pm$0.020 & 0.346$\pm$0.067 \\
KDLOR & 0.295$\pm$0.028 & 0.375$\pm$0.037 & 0.429$\pm$0.017 & 0.584$\pm$0.020 & 0.900$\pm$0.008 & 0.789$\pm$0.012 & 0.377$\pm$0.055 \\
POM & 0.347$\pm$0.029 & 0.406$\pm$0.029 & 0.667$\pm$0.013 & \textbf{0.653$\pm$0.022} & \textbf{0.942$\pm$0.009} & 0.689$\pm$0.021 & 0.428$\pm$0.018 \\
SVR & 0.322$\pm$0.031 & 0.625$\pm$0.030 & 0.476$\pm$0.017 & 0.584$\pm$0.021 & 0.917$\pm$0.006 & 0.845$\pm$0.018 & 0.362$\pm$0.022 \\
NNOR & 0.333$\pm$0.027 & 0.656$\pm$0.041 & 0.619$\pm$0.015 & 0.644$\pm$0.019 & 0.900$\pm$0.005 & 0.816$\pm$0.024 & \textbf{0.442$\pm$0.031} \\
ELMOR & \textbf{0.352$\pm$0.029} & 0.594$\pm$0.038 & 0.571$\pm$0.018 & 0.584$\pm$0.020 & 0.825$\pm$0.019 & 0.816$\pm$0.021 & 0.424$\pm$0.028 \\
CNNOR & 0.344$\pm$0.036 & 0.381$\pm$0.032 & 0.598$\pm$0.011 & 0.581$\pm$0.022 & 0.933$\pm$0.008 & 0.724$\pm$0.039 & 0.318$\pm$0.035 \\
OFMHS & 0.323$\pm$0.039 & 0.491$\pm$0.032 & 0.638$\pm$0.012 & 0.588$\pm$0.014 & 0.917$\pm$0.005 & 0.793$\pm$0.025 & 0.301$\pm$0.029 \\
XOFM & 0.348$\pm$0.023 & \textbf{0.719$\pm$0.031} & \textbf{0.682$\pm$0.021} & 0.637$\pm$0.017 & \textbf{0.942$\pm$0.009} & \textbf{0.858$\pm$0.033} & 0.395$\pm$0.024 \\
\hline
\end{tabular}
}
\end{center}
\end{table*}

\begin{table*}
\begin{center}
{\caption{Test $MAE$ results for each dataset and approach.}\label{table-maeresult}}
\resizebox{0.95\textwidth}{!}{ 
\begin{tabular}{llllllll}
\hline
Approaches & AO & AR & BR & BHO & CU & SO & SK \\
\hline
LIT & 1.256$\pm$0.097 & 0.594$\pm$0.071 & 0.524$\pm$0.065 & 0.475$\pm$0.053 & 0.217$\pm$0.011 & 0.295$\pm$0.042 & 0.964$\pm$0.079 \\
LAT & 1.081$\pm$0.103 & 0.875$\pm$0.081 & 0.619$\pm$0.057 & 0.475$\pm$0.043 & 0.200$\pm$0.018 & 0.305$\pm$0.044 & 0.826$\pm$0.077 \\
KDLOR & 1.268$\pm$0.067 & 0.906$\pm$0.062 & 0.669$\pm$0.053 & 0.485$\pm$0.045 & 0.101$\pm$0.008 & 0.216$\pm$0.029 & 0.856$\pm$0.067 \\
POM & 1.103$\pm$0.078 & 0.688$\pm$0.082 & 0.429$\pm$0.052 & 0.416$\pm$0.037 & \textbf{0.058$\pm$0.009} & 0.316$\pm$0.026 & 0.712$\pm$0.055 \\
SVR & \textbf{0.982$\pm$0.089} & 0.375$\pm$0.063 & 0.619$\pm$0.042 & 0.436$\pm$0.032 & 0.083$\pm$0.005 & 0.179$\pm$0.018 & 0.839$\pm$0.050 \\
NNOR & 1.034$\pm$0.083 & 0.406$\pm$0.042 & 0.381$\pm$0.030 & 0.426$\pm$0.041 & 0.100$\pm$0.013 & 0.184$\pm$0.022 & \textbf{0.685$\pm$0.058} \\
ELMOR & 0.996$\pm$0.093 & 0.438$\pm$0.035 & 0.667$\pm$0.036 & 0.475$\pm$0.030 & 0.175$\pm$0.018 & 0.184$\pm$0.021 & 0.694$\pm$0.033 \\
CNNOR & 1.326$\pm$0.029 & 0.681$\pm$0.013 & 0.488$\pm$0.020 & 0.497$\pm$0.019 & 0.092$\pm$0.007 & 0.298$\pm$0.028 & 1.109$\pm$0.051 \\
OFMHS & 1.398$\pm$0.041 & 0.619$\pm$0.040 & 0.420$\pm$0.021 & 0.502$\pm$0.020 & 0.108$\pm$0.008 & 0.282$\pm$0.016 & 1.182$\pm$0.031 \\
XOFM & 1.033$\pm$0.021 & \textbf{0.343$\pm$0.019} & \textbf{0.363$\pm$0.009} & \textbf{0.411$\pm$0.016} & \textbf{0.058$\pm$0.008} & \textbf{0.152$\pm$0.018} & 0.801$\pm$0.032 \\
\hline
\end{tabular}
}
\end{center}
\end{table*}

We report $Acc$ and $MAE$ in Table \ref{table-accresult} and Table \ref{table-maeresult}, respectively. The best result for each dataset is highlighted. From the $Acc$ results, the proposed XOFM achieves either the best (AR, BR, CU, and SO) or near-the-best (BHO) results for the small-sized datasets. Although EFOM does not perform the best on some large-sized datasets (AO and SK), its performance is better than most baselines and not very much worse than the best one. Similar conclusions can be obtained given Table \ref{table-maeresult}. The proposed XOFM is lowest for five out of seven datasets according to $MAE$. This indicates that the wrong predictions made by XOFM are not much deviated from the true labels. For example, XOFM is not the best method for BHO dataset according to $Acc$; however, it achieves the least $MAE$. It indicates that the proposed XOFM can better take advantage of the ordinal information to reduce the error.

The traditional ordinal regression methods usually require to calculate the difference between two objects, which leads to some sparse training data. This problem heavily affects the performance on smaller datasets given insufficient training samples and larger variance. Similarly, XOFM discretizes an attribute's value into multiple scales in the form of an attribute vector and trains the parameters by determining the differences between every two attribute vectors. This process also leads to the sparsity problem. XOFM utilizes the FMs scheme to address this issue. The experiments results validate the effectiveness of XOFM on five smaller-sized datasets. In the next subsection, we will show that the proposed XOFM can provide meaningful explanations for predictions.

\subsection{Explainability}

The capability to decipher the detailed relationships between attributes and predictions is the key to explain the model. To demonstrate such explainability of XOFM, we present the obtained score functions for the Breast data in Figure \ref{fig-attr}. This dataset contains some electrical impedance measurements in samples of freshly excised tissue from the breast. XOFM is used to classify a sampled tissue into one of six ordered classes, \textit{Carcinoma} $\succ$ \textit{Fibro-adenoma} $\succ$ \textit{Mastopathy} $\succ$ \textit{Glandular} $\succ$ \textit{Connective} $\succ$ \textit{Adipose}, where \textit{Carcinoma} is the most severe class and \textit{Adipose} is the safest. The involved attributes are described in Table \ref{table-desp}.

\begin{figure}
\centerline{\includegraphics[width=1.0\columnwidth]{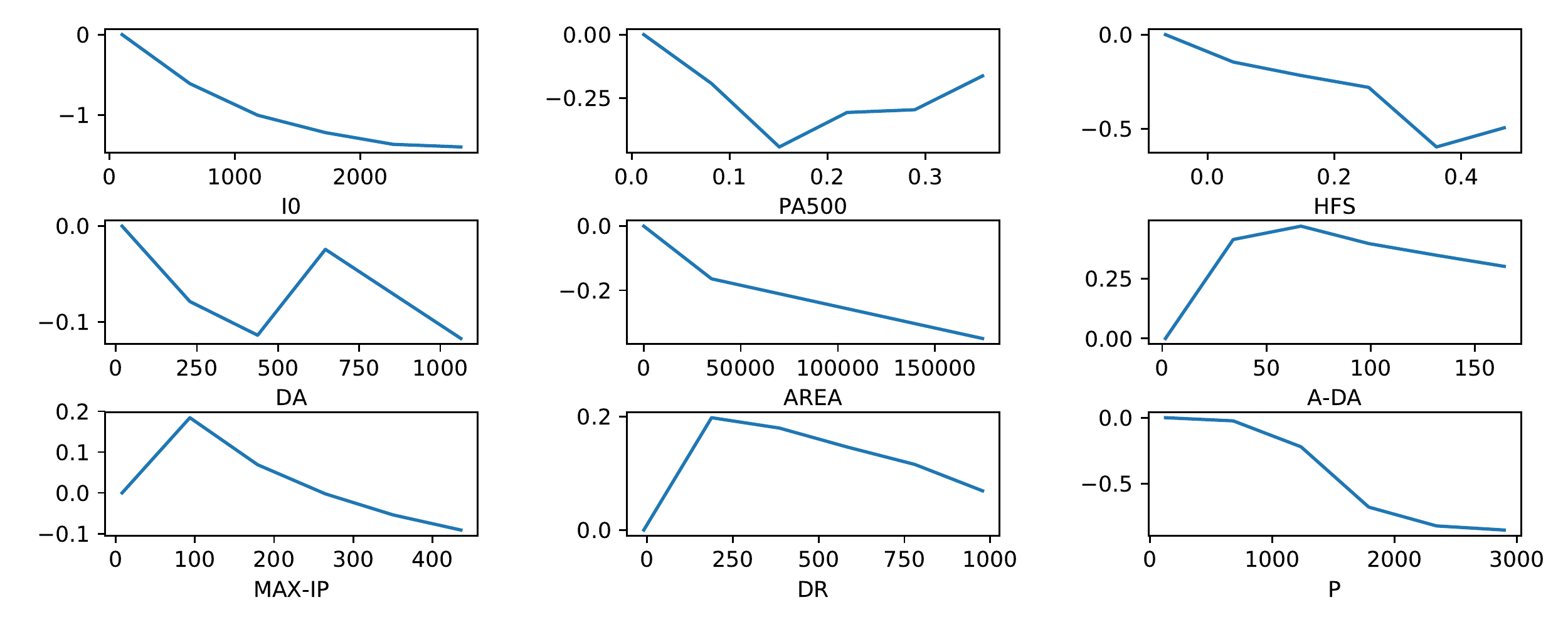}}
\caption{Score functions of the attributes in Breast dataset when $\lambda_1 = 0$ and $\lambda_2 = 0$. }
\label{fig-attr}
\end{figure}

\begin{figure}
\centerline{\includegraphics[width=1.0\columnwidth]{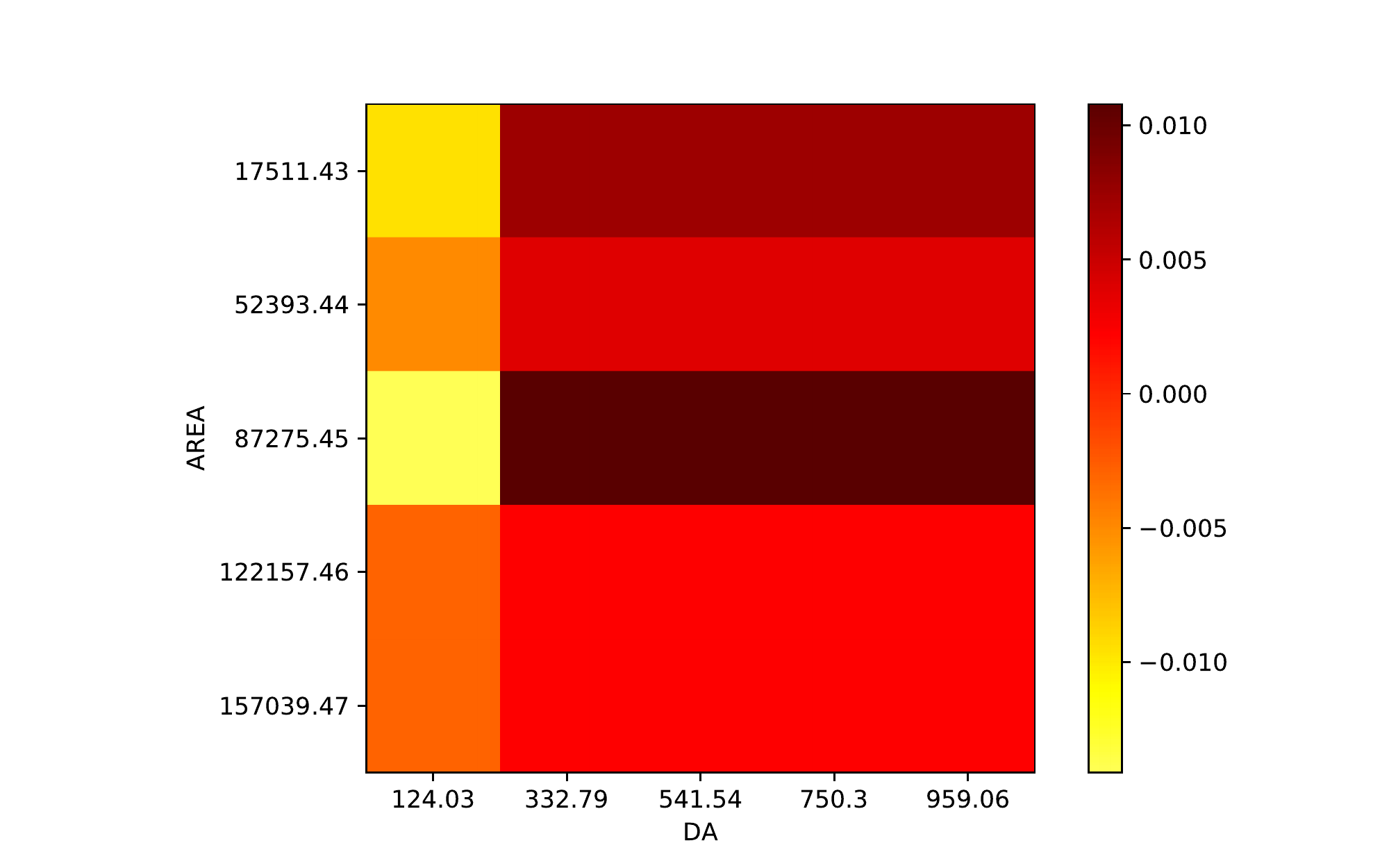}}
\caption{Heat-map for pairwise interactions between attributes \textit{DA} and \textit{AREA} when $\lambda_1 = \lambda_2=0$.}
\label{fig-interact1}
\end{figure}

\begin{table}
\begin{center}
{\caption{Descriptions for attributes in breast data.}\label{table-desp}}
\begin{tabular}{ll}
\hline
Attr. & Description \\
\hline
\textit{I0} & impedivity (ohm) at zero frequency \\
\textit{PA500} & phase angle at 500 KHz \\
\textit{HFS} & high-frequency slope of phase angle \\
\textit{DA} & impedance distance between spectral ends \\
\textit{AREA} &	area under spectrum \\
\textit{A-DA} &	area normalized by \textit{DA} \\
\textit{MAX-IP} & maximum of the spectrum \\
\textit{DR} & distance between \textit{I0} 
and max frequency point \\
\textit{P} & length of the spectral curve \\
\hline
\end{tabular}
\end{center}
\end{table}

In Figure \ref{fig-attr}, each score function is in a piece-wise linear form and characterizes a single attribute's contribution to the risk for breast cancer. For example, the attribute \textit{I0}, the Impedivity (ohm) at zero frequency, negatively affects the risk for an object belonging to the \textit{Carcinoma} class because the marginal score decreases along with the increase of attribute value. Moreover, from the marginal scores scales (the distance between the maximal and the minimal values in $y$-axis), we find that the attribute \textit{I0} is the most important one because its marginal score ranges from 0 to -1.25, which is the largest one among all attributes. That conclusion is consistent with previous clinical findings \cite{da2000classification}.

XOFM can decipher the interaction effects of the attributes. For brevity, we report one of the pairwise interactions as a heat-map (Figure \ref{fig-interact1}). The color represents the intensity of the interactions in different attribute value intervals. For example, when \textit{AREA} is around 87,275.45, its interaction with \textit{DA} is stronger, indicating that a breast tissue with $AREA$ around 87,275.45 and $DA$ within the interval $[332.79, 959.06]$ is more likely to be malignant.

\subsection{Modification}
\label{subsec-modification}

XOFM is an flexible model that can be modified by tuning the regularization terms. As introduced in previous section, $\lambda_1$ controls the complexity of the score functions and $\lambda_2$ determines the intensity of the attribute interactions. In addition to determine the parameters by cross-validation process, we can progressively adjust the parameters based on the user's domain knowledge. For example, if a physician insists that the main effects are more important than the attribute interactions, we can increase $\lambda_2$ and the resulted new heat-map is shown in Figure \ref{fig-interact2}. Obviously, the intensity of the interaction is weaker than the previous one. On the contrast, if we increase $\lambda_1$, the score functions would become more flat and some exhibit different curve shapes (Figure \ref{fig-attr2}). In practice, we can determine the values of regularization terms based on the performance, or by soliciting the opinion of domain experts.
\begin{figure}
\centerline{\includegraphics[width=1.0\columnwidth]{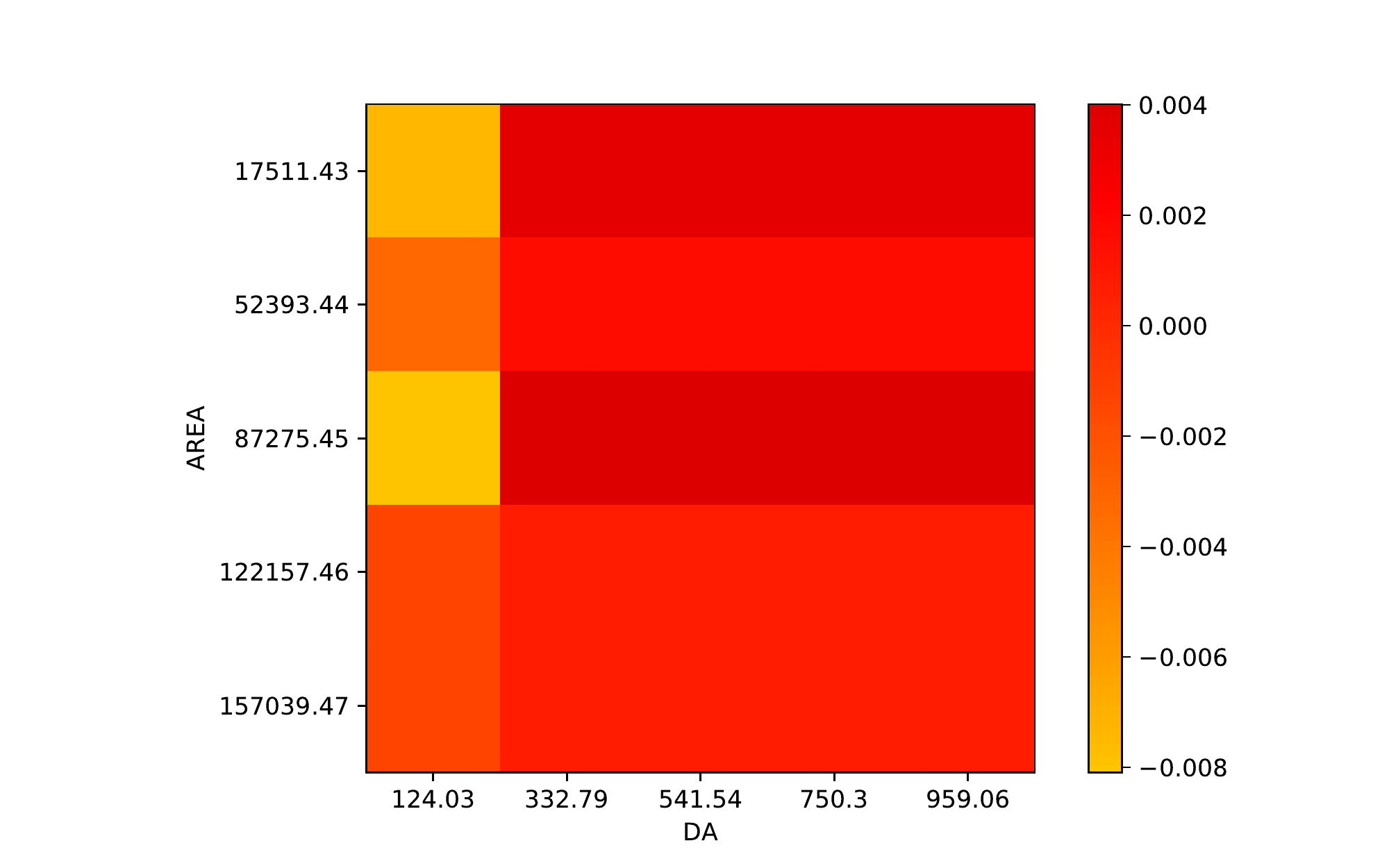}}
\caption{Heat-map for pairwise interactions between attributes \textit{DA} and \textit{AREA} when $\lambda_1 = 0$ and $\lambda_2=0.005$.}
\label{fig-interact2}
\end{figure}

\begin{figure}
\centerline{\includegraphics[width=1.0\columnwidth]{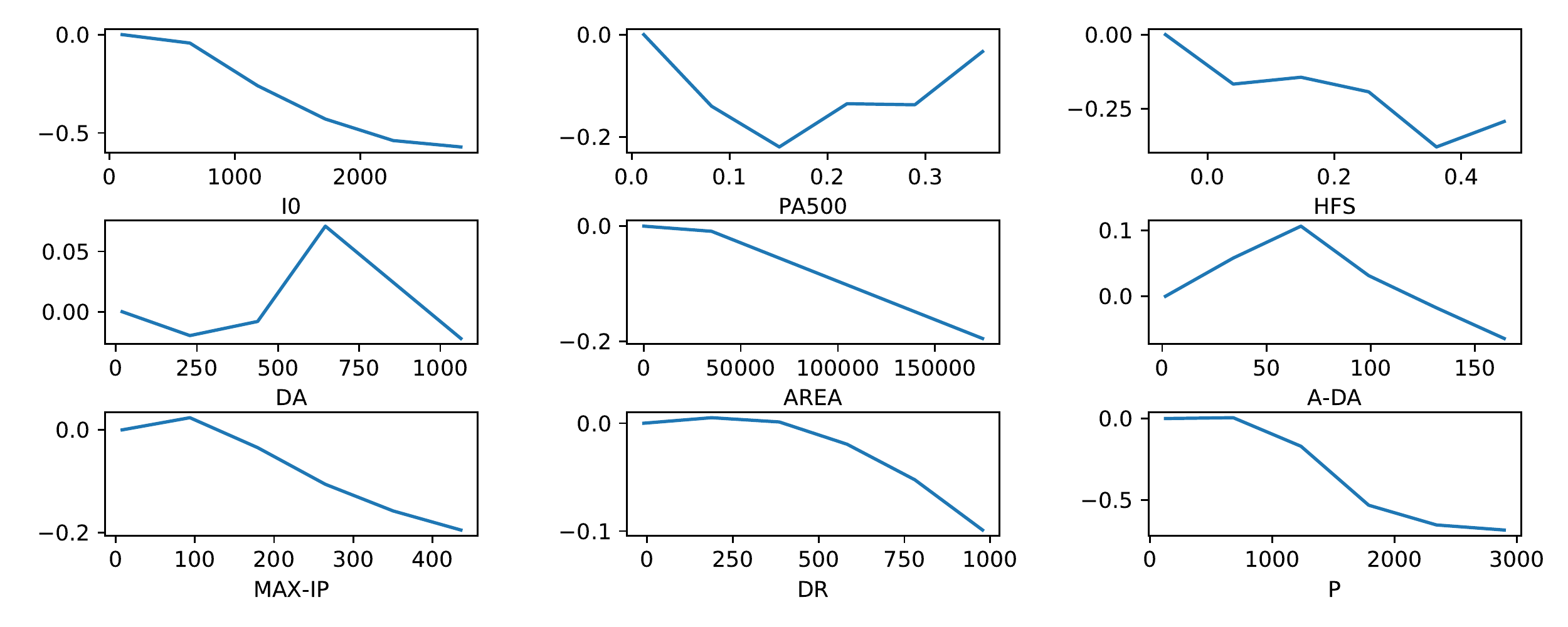}}
\caption{score functions of the attributes in Breast dataset when $\lambda_1 = 0.01$ and $\lambda_2 = 0$.}
\label{fig-attr2}
\end{figure}

\section{Conclusion}
\label{sec-con}

In this study, we propose the XOFM, a new factorization machines-based ordinal regression model. XOFM is able to provide state-of-the-art prediction performance, and more importantly, provide meaningful explainability that deciphers the detailed contributions of attributes and their interactions. Such explainability makes XOFM uniquely effective in providing decision making support, where the ability to explain how the predictions are made is as much needed as achieving good accuracy. Moreover, XOFM presents a general explainable-modeling framework that can be calibrated/modified for various prediction problems other than ordinal regressions. 

\ack We appreciate the anonymous reviewers and PC members for their valuable suggestions and comments.

\bibliography{ecai}



\section*{Supplementary material}

# Supplementary Materials for 'Explainable Ordinal Factorization Model: Deciphering the Effects of Attributes by Piece-wise Linear Approximation'

Mengzhuo Guo[a,b], Zhongzhi Xu[b], Qingpeng Zhang[b,*], Xiuwu Liao[a], Jiapeng Liu[a]

*[a]School of Management, Xi'an Jiaotong University*
*[b]School of Data Science, City University of Hong Kong*

**Abstract**

This file contains supplementary materials for the submission 'Explainable Ordinal Factorization Model: Deciphering the Effects of Attributes by Piece-wise Linear Approximation'.

*Keywords:* Supplementary Materials

**1. Experimental results for the proposed XOFM with monotonicity constraints.**

We validate the performance of the proposed XOFM with monotonicity constraints (Section 3.5) by applying it to evaluating university. The dataset contains the best Chinese University (CU) rankings in 2018[1]. The involved attributes are described in Table 1. Apparently, all attributes have monotonically increasing preference directions by their nature, thus we apply the proposed XOFM with monotonicity constraints and optimize the parameters using the adapted algorithm introduced in Section 3.5. To compare the results, we also select a traditional ordinal regression model UTADIS that is widely used to handle monotonic constraints as a baseline model [1]. Following the rules in our study, we first use a five-fold

*[*]Corresponding author
  *Email addresses:* guomengzhuo@stu.xjtu.edu.cn (Mengzhuo Guo),
zhongzhxu2-c@my.cityu.edu.hk (Zhongzhi Xu),
qingpeng.zhang@cityu.edu.hk (Qingpeng Zhang),
liaoxiuwu@mail.xjtu.edu.cn (Xiuwu Liao), jiapengliu@mail.xjtu.edu.cn
(Jiapeng Liu)
  [1]http://www.shanghairanking.com/Chinese_universities_Rankings/Overall-Ranking-2018.html

cross validation process to examine the performances and report the average $Acc$ and $MAE$ results. At last, we use all data to learn the marginal value functions and interaction effects.

Table 1: Descriptions for attributes in Chinese University data.

| Attr. | Description |
| --- | --- |
| $g_1$ Quality of incoming students | Average scores of freshmen in National College Entrance Exam |
| $g_2$ Employment rate | The employment rate of bachelors |
| $g_3$ Reputation | Donations from sociality |
| $g_4$ Scale of research | Number of publications in Scopus |
| $g_5$ Quality of research | Field weighted citation impact |
| $g_6$ Top research achievements | 1% most cited publications in the world |
| $g_7$ Top scholars | Chinese most cited researchers |
| $g_8$ Technology service | Income from industry |
| $g_9$ Technology transfer | Income from technology transfer |
| $g_{10}$ Internationality | Ratio of international students |

Table 2 presents the $Acc$ and $MAE$ results given different predefined numbers of sub-intervals. We observe that although the adapted XOFM constrains the monotonicity of the attributes, it achieves high predictive performances. Although a larger number of predefined sub-intervals increases the model complexity, it may cause data sparsity, thereby decreasing the model performance (see the result obtained by UTADIS when $\gamma$ changes from 4 to 5). The adapted XOFM can reduce the impact of sparse data and perform better than UTADIS when $\gamma$ is larger.

Table 2: The average results for CU dataset.

| $\gamma_j$ | 1 | 2 | 3 | 4 | 5 |
| --- | --- | --- | --- | --- | --- |
| $Acc$ (XOFM) | 0.826 | 0.909 | 0.909 | **0.950** | **0.950** |
| $MAE$ (XOFM) | 0.182 | 0.091 | 0.091 | **0.049** | **0.049** |
| $Acc$ (UTADIS) | 0.840 | 0.901 | 0.917 | **0.934** | 0.926 |
| $MAE$ (UTADIS) | 0.160 | 0.009 | 0.083 | **0.066** | 0.074 |

Then, we use the whole data as the training data to derive marginal value functions and interaction effects, and depict them in Figures 1 and 2, respectively. In Figure 1, the obtained marginal value functions are monotonic as required and tend to be smooth because the $\lambda_1$ regularizes the changes of marginal values of two consecutive characteristic points. Moreover, it can be observed that attribute $g_1$ (Quality of incoming students) is the most important one because it has the maximal share in the marginal values among all attributes. Figure 2 presents interaction effects of all pairwise interacting attributes. The positive value means that

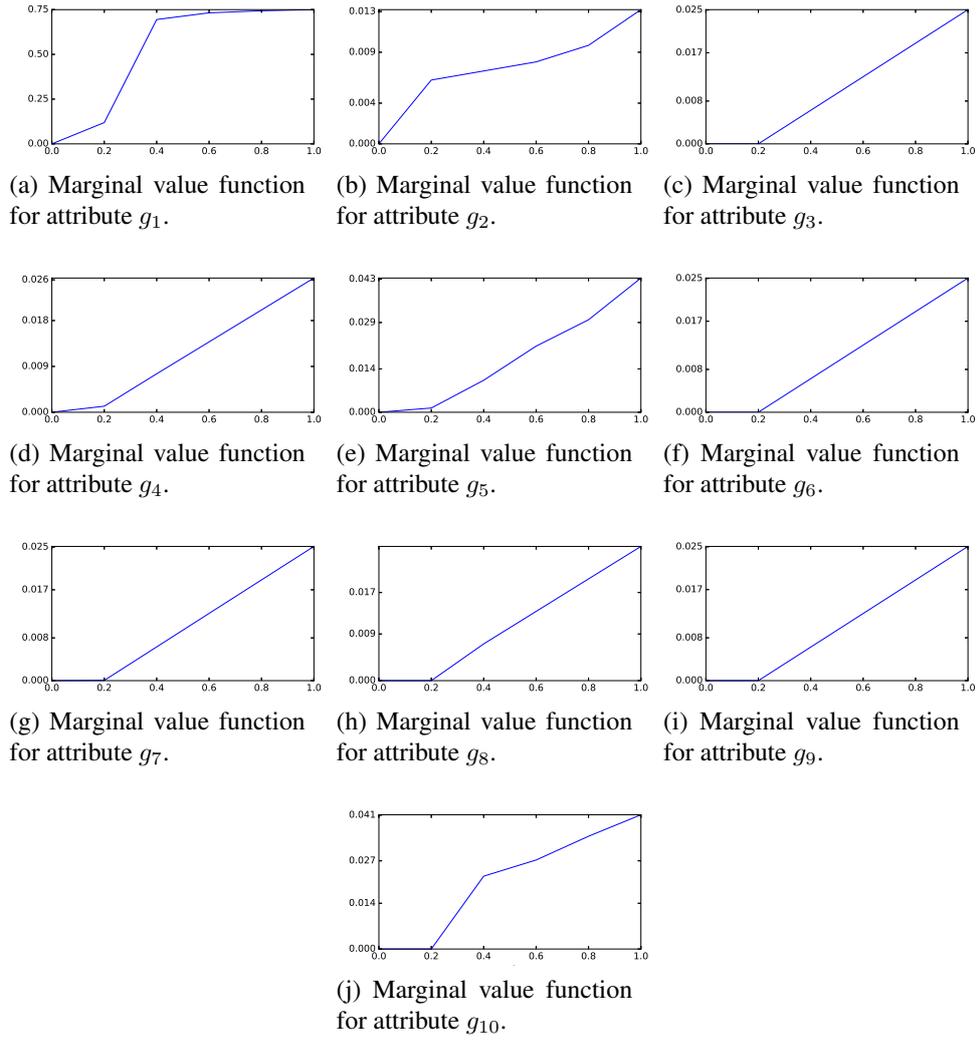

Figure 1: The obtained normalized marginal value functions for Chinese University dataset.

if we keep one attribute value unchanged, their interaction will positively affect the global value along with the increase of the other attribute, thereby maintaining the assumption that the attributes are monotonic. Moreover, the color represents the intensity of such interaction effect of the corresponding sub-intervals. From the figure, not all pairwise attributes have strong co-effects on the global value. Many of them have very weak effects on most of sub-intervals, but the attribute $g_1$ has stronger interacting effects with other attributes due to its higher importance explored in Figure 1. Another interesting pattern is that such interacting effects are usually enhanced in some specific intervals and then are declined. The reason is that we use the project gradient descent algorithm to satisfy the assumption of monotonicity, and the effects on the parameters are similar to the $L_0$ normalization, thereby causing most of the interacting effects being zero.

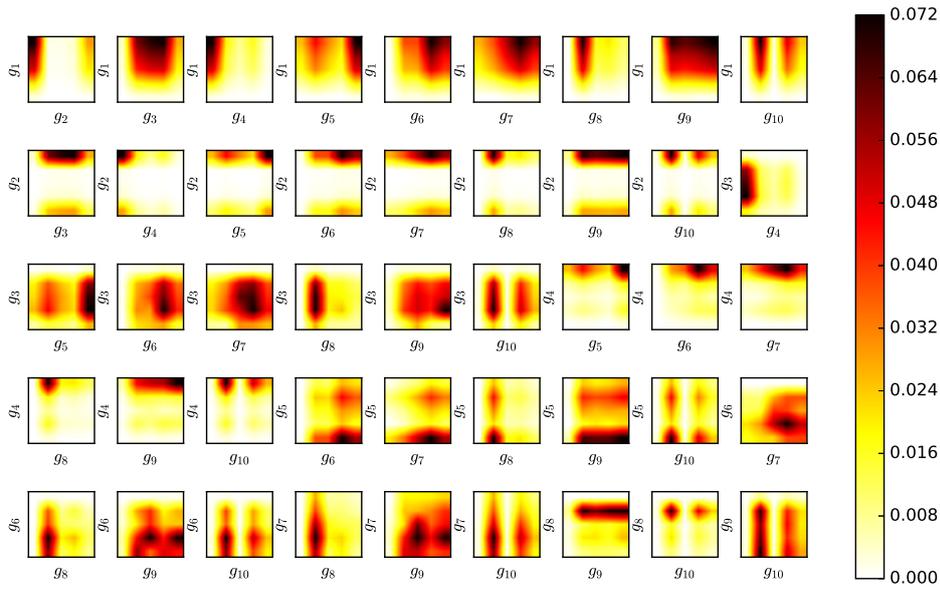

Figure 2: Heat-maps for interaction effects of all pairwise attributes.



\end{document}